\documentclass[journal, 10 pt, twoside, dvipsnames]{IEEEtran} %tran
\usepackage{graphicx}
\usepackage{float}
\usepackage{subfigure}

\usepackage{lscape}
\usepackage{rotating}
\usepackage{algorithm}
\usepackage{algpseudocode}
\usepackage{comment}
\usepackage{multirow}
\usepackage{multicol}
\usepackage{booktabs}
\usepackage{paralist}

\usepackage{enumitem}

\usepackage{amssymb,amsmath,amsfonts}
\usepackage{array}
\usepackage{bm}
\usepackage{commath}
\usepackage{savesym}
\usepackage{units}
\usepackage{dsfont}
\PassOptionsToPackage{usenames,dvipsnames}{xcolor}
\usepackage{soul}
\usepackage{xspace}
\usepackage{setspace}
\usepackage{afterpage}
\usepackage{microtype}
\usepackage{silence}
\usepackage[noadjust]{cite}
\makeatletter
\let\NAT@parse\undefined
\makeatother
\usepackage[hidelinks]{hyperref}
\usepackage{cleveref}
\usepackage[all]{hypcap}
\usepackage{tikz}

\newcommand{\trs}{\top}

\newcommand{\RR}{\mathbb R}

\newcommand{\B}{\mathcal B}
\newcommand{\D}{\mathcal D}
\newcommand{\E}{\mathcal E}

\newcommand{\I}{\mathcal I}

\newcommand{\M}{\mathcal M}

\renewcommand{\O}{\mathcal O}

\newcommand{\R}{\mathcal R}
\renewcommand{\S}{\mathcal S}

\newcommand{\one}{{\mathds 1}}

\graphicspath{{Figures/}}
\crefname{equation}{}{}
\renewcommand{\algorithmiccomment}[1]{\bgroup\hfill$\triangleright$~#1\egroup}
\DeclareMathOperator*{\argmax}{arg\,max}
\newtheorem{theorem}{Theorem}
\newtheorem{problem}{Problem}
\newtheorem{remark}{Remark}[section]
\newtheorem{definition}{Definition}
\newtheorem{proof}{Proof}
\renewenvironment{proof}{\begin{IEEEproof}}{\end{IEEEproof}\ignorespacesafterend}
\renewcommand{\IEEEproofindentspace}{1\parindent}

\newcommand{\etc}{\textit{etc}}
\newcommand{\ie}{\textit{i}.\textit{e}.\,}
\newcommand{\eg}{\textit{e}.\textit{g}.\,}
\renewcommand{\secref}[1]{Section~\ref{#1}}
\renewcommand{\figref}[1]{Fig.~\ref{#1}}
\newcommand{\tabref}[1]{Table~\ref{#1}}
\newcommand{\defref}[1]{Definition~\ref{#1}}

\renewcommand{\algref}[1]{Algorithm~\ref{#1}}

\title{%\LARGE \bf
Submodular Optimization for Coupled Task Allocation and Intermittent Deployment Problems}

\author{Jun Liu and Ryan K. Williams
\thanks{This work was supported by the National Institute of Food and Agriculture under Grant 2018-67007-28380.}
\thanks{The authors are with the Department of Electrical and Computer Engineering, Virginia Polytechnic Institute and State University, Blacksburg, VA 24061 USA (e-mail: \href{mailto:junliu@vt.edu}{junliu@vt.edu}; \href{mailto:rywilli1@vt.edu}{rywilli1@vt.edu}).}
}

\begin{document}

\bstctlcite{IEEEexample:BSTcontrol}
\maketitle
% \IEEEoverridecommandlockouts

\pagestyle{empty}
\thispagestyle{empty}

\begin{abstract}
    In this letter, we demonstrate a formulation for optimizing \emph{coupled} submodular maximization problems with provable sub-optimality bounds. In robotics applications, it is quite common that optimization problems are coupled with one another and therefore cannot be solved independently.  Specifically, we consider two problems coupled if the outcome of the first problem affects the solution of a second problem that operates over a longer time scale. For example, in our motivating problem of environmental monitoring, we posit that multi-robot task allocation will potentially impact environmental dynamics and thus influence the quality of future monitoring, here modeled as a multi-robot intermittent deployment problem. The general theoretical approach for solving this type of coupled problem is demonstrated through this motivating example. Specifically, we propose a method for solving coupled problems modeled by submodular set functions with matroid constraints. A greedy algorithm for solving this class of problems is presented, along with sub-optimality guarantees. Finally, practical optimality ratios are shown through Monte Carlo simulations to demonstrate that the proposed algorithm can generate near-optimal solutions with high efficiency.
\end{abstract}


\section{Introduction}

It is common that multi-robot team objectives are intertwined or coupled, with an especially interesting example being objectives that operate sequentially over different time scales. Consider, for example, an environmental monitoring application where we first need to allocate a group of heterogeneous robots to perform a set of tasks, \eg, collecting samples or otherwise interacting with the environment. This problem is well-known as a multi-robot task allocation problem and occurs over a short time scale. However, the critical factor considered in this letter is that the tasks themselves may impact underlying environmental dynamics, and thus future long-term objectives will be influenced. Here we consider the future long-term objective of multi-robot \emph{intermittent deployment}, where we ask: \emph{when is it appropriate to deploy a robotic team for long-term monitoring?} To account for the impact of short-term task allocation on long-term monitoring, we formulate a general coupled submodular optimization problem, which yields a bounded sub-optimal solution for a provably hard problem.

Multi-robot task allocation problems have been studied for a long time \cite{gerkey2004formal,korsah2013comprehensive}. In this letter, our focus will instead be on the intermittent deployment problem and its coupling with task allocation. The idea of intermittence in robotics applications can be either by design or intrinsic.  In our motivating example of environmental monitoring, we design the system to intermittently deploy to reduce cost over the long-term. In \cite{kantaros2017distributed}, the mobile robot networks are required to be connected intermittently, which is achieved through a linear temporal logic method. In \cite{hollinger2010multi}, the robots can only communicate periodically in predefined time steps. On the other hand, in \cite{sinopoli2004kalman}, the authors studied the convergence of Kalman filtering when the measurement arrival time is intrinsically intermittent. A similar application to our deployment problem without the intermittent feature is the sensor scheduling problem \cite{jawaid2015submodularity}, which needs to schedule sensors sequentially to estimate a linear system.

More generally, various multi-robot problems contain two or more sub-problems coupled in some way. In multi-robot motion planning applications, such as collaborative coordination \cite{saha2006multi}, the movement of one robot will impact the others. In the environmental monitoring problem \cite{kalra2007generalized}, different robots need to work together to cover/explore the environment more efficiently.  In humanoid robot manipulation \cite{gienger2008task}, the problem of finding an optimal grasp position and the problem of reaching the object is indeed strongly coupled. Thus, in general, it is necessary to model problem couplings and seek efficient algorithms to provide quality solutions. If the domain of a problem is discrete, we need to consider methods for combinatorial optimization, as we do in this work. For example, the sensor placement problem \cite{krause2008near} seeks to find locations for sensors from a discrete location set to maximize the mutual information for estimating the environment. The abstract task allocation problem seeks to assign robots to tasks to maximize the reward \cite{williams2017decentralized}. Other applications can be found in the target tracking problem, the environmental monitoring problem, \etc. Generally, these problems are NP-hard \cite{feige1998threshold} and cannot be solved optimally by using a polynomial time algorithm. Moreover, if two or more problems are coupled with each other, it is even harder to have quality solutions. Therefore, previous work mainly focuses on generating approximation methods which yield sub-optimality bounds that are often used in practice. In particular, the key focus of late is on greedy algorithms for submodular function optimization. These greedy algorithms usually come with a performance guarantee.
The first provable bound for submodular function optimization over general matroid constraints is shown in \cite{fisher1978analysis}.
Combining the modular and the submodular result as a single result, submodular curvature is used in \cite{conforti1984submodular,sviridenko2017optimal,tzoumas2018resilient} for proving optimality bounds. A recently improved version of monotone submodular function maximization over general matroid constraints by using a multi-linear relaxation scheme is shown in \cite{chekuri2014submodular}. Also, \cite{santiago2016multi} proposed a multivariate version of submodular optimization using the multi-linear extension. This work focuses on minimizing or maximizing a single objective represented as a multivariate function. In our work, the objective function includes two sub-objective functions. For a comprehensive overview of submodular optimization, the reader is referred to \cite{krause14survey} for additional details.

A common practice to solve coupled problems is to solve each problem separately and combine solutions. In this letter, we instead propose to solve coupled submodular optimization problems with general matroid constraints. As an example of such a problem, we couple a task allocation problem \cite{williams2017decentralized} with an intermittent deployment problem \cite{liu2018optimal} where robots are optimally deployed to monitor an environment over time.

\emph{Contributions:} In summary, the contributions of this letter are as follows:
\begin{enumerate}
    \item We formalize a modeling and solution method for coupled submodular optimization problems with general matroid constraints.

    \item We demonstrate how to use matroids to model constraints in robotic applications;

    \item We provide a greedy algorithm with bounded optimality for solving the general coupled optimization problem. We demonstrate this by using a combination of the task allocation problem and the intermittent deployment problem to show the performance in an environmental monitoring application.
\end{enumerate}

\emph{Organization:} The remainder of this letter is organized as follows. We first introduce the preliminaries and the problem formulation in \secref{sec: preliminaries}. In \secref{sec: main}, we present the details about our running example: the multi-robot task allocation problem and the multi-robot intermittent deployment problem. Then, we generalize the properties of our coupled problem formulation. In \secref{sec: alg}, we present a greedy algorithm with provable performance bounds. In \secref{sec: simulations}, we demonstrate the result of the proposed algorithm using Monte Carlo simulations. Conclusions and directions for future work are stated in the final section.

\section{Preliminaries and Problem Formulation}
\label{sec: preliminaries}

As the multi-robot task allocation problem and the multi-robot intermittent deployment problems considered in this letter are discrete in nature, we start with basic definitions related to discrete optimization.

\subsection{Submodular Function Optimization}

A set function \cite{schrijver2003combinatorial} $f: 2^V \mapsto \RR$ is a function that assigns each subset $A \subseteq V$ a value $f(A) \in \RR$, where $V$ is a finite set called the ground set. If $A=(a_1, \ldots, a_n)$ is a sequence and $f: 2^V \mapsto \RR$, then $f(\cdot)$ is a sequence function. Note that different sequences will generate different objective values. For example, if $A_1 = (a_1, a_2)$ and $A_2 = (a_2, a_1)$, then $f(A_1) \neq f(A_2)$ when $f(\cdot)$ is a sequence function. In this letter, we only consider the case of set functions and sequence functions with finite ground sets.
Next, we review set function properties.

\begin{definition}[\cite{schrijver2003combinatorial}]
    A set function $f: 2^{V} \mapsto \RR$ with $V$ as the ground set is
    \begin{itemize}
        \item \emph{normalized}, if $f(\emptyset) = 0$.
        \item \emph{non-decreasing}, if $f(A) \le f(B)$ for all $A \subseteq B \subseteq V$.
        \item \emph{modular}, if $f(A) = \sum_{a \in A} f(a)$ for all $A \subseteq V$.
        \item \emph{submodular}, if $f(A) + f(B) \ge f(A \cup B) + f(A \cap B)$ for all $A, B \subseteq V$.
    \end{itemize}
    \label{def: basic}
\end{definition}

In this letter, we restrict our discussion to normalized functions because an unnormalized function $f: 2^V \mapsto \RR$ can be normalized to $f'(A)$ as $f'(A) = f(A) - f(\emptyset)$.
Another equivalent definition of submodular set function \cite{schrijver2003combinatorial} is that $f(A \cup \{e\}) - f(A) \ge f(B \cup \{e\}) - f(B)$ holds for any $A \subseteq B \subseteq V$ and $e \in V \setminus B$ with $V$ as the ground set for the set function $f: 2^V \mapsto \RR$. This property is called a diminishing return property since the marginal gain $f(\{e\} | A) = f(A \cup \{e\}) - f(A)$ becomes less when $A$ is replaced by a larger set $B$, \ie, $f(\{e\} | A) \ge f(\{e\} | B)$. Examples of non-decreasing submodular functions include:
\begin{itemize}
    \item $f(A) = \max_{i \in A} w_i$ with $A \subseteq V$ and $w_i \ge 0$;
    \item $f(A) = | \bigcup_{i \in A} S_i|$ with $A \subseteq V$ and $S_i \subset V$;
    \item $f(A) = \min \{\sum_{i \in A} w_i, b \}$ with $A \subseteq V$, $w_i \ge 0$, $b \ge 0$.
\end{itemize}
In \defref{def: basic}, if we replace the set function with a sequence function, we can define properties similar to the above. Specifically, if the modularity or submodularity holds for a sequence function, we call it a sequence modular function or a sequence submodular function.

\begin{definition}[\cite{oxley2006matroid}]
    \label{def: matroid}
    A matroid $\M = (V, \I)$ is a pair $(V, \I)$, where $V$ is a finite set (called the ground set) and $\I$ is a collection of subsets of $V$, with the following properties:
    \begin{enumerate}[label=\roman*)]
        \item $\emptyset \in \I$; \label{item: matroid1}
        \item If $X \subseteq Y \in \I$, then $X \in \I$; \label{item: matroid2}
        \item If $X, Y \in \I$ with $|Y| < |X|$, then there exists an element $x \in X \setminus Y$ such that $Y \cup \{x\} \in \I$. \label{item: matroid3}
    \end{enumerate}
\end{definition}

A matroid constraint is a constraint that is represented by admissible subsets of the ground set that satisfy the above axioms. For example, given a ground set $V$, if $\I$ is a collection of subsets of $V$ that are at most size $\ell$, \ie, $\I = \{ A \subseteq V: |A| \le \ell \}$, then $\M = (V, \I)$ is a matroid constraint as Definition \ref{def: matroid} is respected.  In multi-robot task allocation, this simple matroid example could constrain each robot $i$ to choose at most $\ell$ tasks from a ground set. Examples of matroid constraints $\M = (V, \I)$ include:
\begin{itemize}
    \item \emph{Uniform matroid:} $M = (V, \I)$ where $\I = \{A \subseteq V: |A| \le \ell \}$. Examples can be found in resource limited applications in robotics, control, \etc. These resources can be batteries, communication bandwidths, computation resources, information about targets \cite{cesare2015multi,kantaros2018distributed,dong2012tracking}, \etc.

    \item \emph{Partition matroid:}  $M = (V, \I)$ where $\I = \{A \subseteq V: |A \cap V_i| \le \ell_i, \forall i = 1, \ldots, n \}$, $V = \bigcup_{i=1}^n V_i$, and $V_i$'s are disjoint. Examples can be found in heterogeneous systems with limited resources in robotics, control, \etc. For example, a robotic system with different types of robots, sensors, batteries, payloads \cite{jorgensen2017matroid,zhang2013amcl,liu2013square,corah2019distributed}, \etc.
\end{itemize}
The interesting aspect of a matroid is its ability to model constraints that allow for efficiently computable solutions, which is especially useful in robotic applications. We will demonstrate examples and details when matroids are used in our problem in the next section.

An important class of problems that combine submodularity and matroid constraints is submodular maximization subject to a matroid constraint. Specifically, in this problem, given a ground set $V$ and a matroid $\M = (V, \I)$, we want to find a subset $S \subseteq V$ to maximize a submodular function $f: 2^V \mapsto \RR$ such that $S$ satisfies all three matroid axioms, \ie, $S \in \I$. If there are $n$ matroid constraints, \ie, $\M_i = (V, \I_i), \forall i = 1, \ldots, n$, that need to be satisfied, we can write it as a matroid intersection constraint $\M = (V, \I)$ with $\I = \bigcap_{i=1}^n \I_i$. The cardinality of this matroid intersection is $|\M| = n$.

\subsection{Problem Formulation}
\label{problem formulation}

\begin{problem}
\label{problem: 1}
The multi-robot task allocation problem coupled with the multi-robot intermittent deployment is given by:
\begin{align*}
    \underset{A \subseteq E}{\text{maximize}} \quad & g(A) + \max_{B \subseteq V} f(A, B) \\
    \text{subject to} \quad                         & A \in \I_1, B \in \I_2.
\end{align*}
where $A$ is a multi-robot task allocation chosen from the finite ground set $E$ with $A$ satisfying the matroid intersection constraint $\M_1 = (E, \I_1)$, \ie, $A \in \I_1$. The function $g: 2^E \mapsto \RR$ is a utility function for the multi-robot task allocation problem. $B$ is a multi-robot deployment policy chosen from the finite ground set $V$ with $B$ satisfying the matroid intersection constraint $\M_2 = (V, \I_2)$, \ie, $B \in \I_2$. The function $f: 2^{E \times V} \mapsto \RR$ is a utility function for the intermittent deployment problem, where $E \times V$ is the Cartesian product of $E$ and $V$. $f(\cdot)$ is a function of both $A$ and $B$ because we assume that multi-robot task allocations \emph{(first phase, short-term)} have an impact on the multi-robot intermittent deployment action \emph{(second phase, long-term)}.
Then, the objective function is
\begin{equation*}
    \begin{split}
        m(A) & = g(A) + h(A) \\
        & = g(A) + \max_{B \subseteq V} f(A, B).
    \end{split}
\end{equation*}
\end{problem}

In the following sections, we will detail how to build our modular/submodular functions, how to use matroids to model constraints, and how to solve Problem \ref{problem: 1} efficiently.

\section{Coupled Multi-Robot Task Allocation and Intermittent Deployment Problem}
\label{sec: main}

Now, we present the details about each problem and then give the properties of the problem formulation.

\subsection{The Multi-Robot Task Allocation Problem}

The multi-robot task allocation model comes from our previous work \cite{williams2017decentralized}. Briefly, the formulation of this problem with matroid intersection constraint is:
\begin{align*}
    \underset{A \subseteq E}{\text{maximize}} \quad & g(A)        \\
    \text{subject to} \quad                         & A \in \I_1,
\end{align*}
where $g: 2^E \mapsto \RR$ is the utility function and $\M_1 = (E, \I_1)$ is the matroid intersection constraint. The element of the ground set $E$ is represented by the \emph{triplet} $(r, d, e)$ for $r \in \R_1, d \in \D, e \in \E$, and $(d, e) \in \O$. Here, $\R_1$ is the robot ground set. $\O$ is the functionality-requirement ground set. $(r,d,e)$ can be read as ``robot $i$ performs functionality $d$ for the requirement $e$''. Each functionality-requirement pair $(d, e)$ is a task. Therefore, $\O$ can also be viewed as the task ground set. Each triplet $(r,d,e)$ forms an element of an allocation set $A$. The goal is to allocate tasks from $\O$ to the robots in $\R_1$ to form an allocation set $A$ to maximize the utility $g(A)$.

To make it more clear, let's look at an example. In the example, we define $\R_1 = \{1, 2\}$, which means there are two robots available. Specifically, let's consider the case that the first robot is a ground robot and the second one is an aerial robot. Also, if functionality-requirement set is $\O = \{ (d_1, e_1), (d_2, e_2) \}$, where $(d_1, e_1)$ means flying ability (functionality: $d_1$) for a long distance package delivery (requirement: $e_1$), and $(d_2, e_2)$ means moving/flying ability (functionality: $d_2$) for data collection (requirement: $e_2$). Then, we can build constraints for these two robots as follows. Specifically, the constraint for robot 1 is $\I_{1}= \{\{(1, d_2, e_2)\}\}$ and the constraint for robot 2 is $\I_{2} = \{ \{(2, d_1, e_1)\}, \{(2, d_2, e_2)\}\}$. We construct these two constraints due to the reason that the aerial robot 2 can finish both $(d_1, e_2)$ and $(d_2, e_2)$ functionality-requirement pairs while the ground robot 1 can only finish the pair $(d_2, e_2)$. This \emph{independence constraint} $\M_{11}$ with two other constraints,  \emph{uniqueness constraint} $\M_{12}$ and \emph{topology constraint} $\M_{13}$, are matroidal as shown in \cite{williams2017decentralized}. The \emph{uniqueness constraint} requires each functionality-requirement pair can only be allocated no more than once. The \emph{topology constraint} requires the distance of adjacent elements of an allocation is less than a threshold to ensure robots can communicate with each other. The intersection of these matroid constraints forms the matroid intersection constraint $\M_1$ of this problem. For the utility function $g(\cdot)$, a simple example would be the sum of reward for each element $a$ of allocation set $A$, \ie, $g(A) = \sum_{a \in A} u_a$, where $u_a$ is the reward of the allocation element $a = (r,d,e)$.

\subsection{The Multi-Robot Intermittent Deployment Problem}
\label{ssec: inter}

The idea of intermittently deploying a multi-robot system is to render the system more efficient by asking: \emph{When is it appropriate to deploy a robotics team? Which combination of robots is suitable?}

As we will deal with deployment constraints over time, we first partition the ground set $V$ of this problem into disjoint sets $V_1, \ldots, V_K$ over a time horizon of $K$ steps. The partition at time $k$ is $V_k$. Specifically, $V_k = \{ (r, d) | r \in \R_2, d \in \{0,1\} \}$ with $\R_2$ the set of robots for this problem and $d$ the deployment action, where $0$ and $1$ means not deploy and deploy, respectively. To capture the idea of intermittent deployment, we require the deployment policy to satisfy the following constraints:

\begin{enumerate}[label=\arabic*)]
    \item No more than $\ell_k$ robots can be deployed at time $k$ for $k = 1, \ldots, K$. This is our constraint $\M_{21}$. \label{item: M21}
    \item The number of times where there is at least one robot deployed is less than or equal to $\ell$. This is $\M_{22}$. \label{item: M22}
    \item Each robot can only be selected or not selected at every time $k$ for $k = 1, \ldots, K$. This is our constraint $\M_{23}$. \label{item: M23}
\end{enumerate}

\noindent To satisfy \ref{item: M21}, consider the constraint $\M_{21} = (V, \I_{21})$ where
\begin{equation}
    \I_{21} = \{ B \subseteq V: \abs{B \cap V_k} \le \ell_k\}.
    \label{eq: M21}
\end{equation}
To satisfy \ref{item: M22}, consider the constraint $\M_{22} = (V, \I_{22})$ where
\begin{equation}
    \I_{22} = \{ B \subseteq V: \sum_{k=1}^K \one( |B \cap V_k| ) \le \ell\},
    \label{eq: M22}
\end{equation}
and $\one(\cdot)$ is an indicator function that takes the form
\begin{equation*}
    \one( |B \cap V_k| ) =
    \begin{cases}
        1 & \text{if $ |B \cap V_k| \ge 1 $,} \\
        0 & \text{if $ |B \cap V_k| = 0 $.}   \\
    \end{cases}
\end{equation*}
To satisfy \ref{item: M23}, consider the constraint $\M_{23} = (V, \I_{23})$, where
\begin{equation}
    \I_{23} = \{B \subseteq V: |B_r \cap V_k| = 1, \forall r \in \R_2 \}
    \label{eq: M23}
\end{equation}

\begin{figure}[!tbp]
    \centering
    \includegraphics[width=0.8\linewidth]{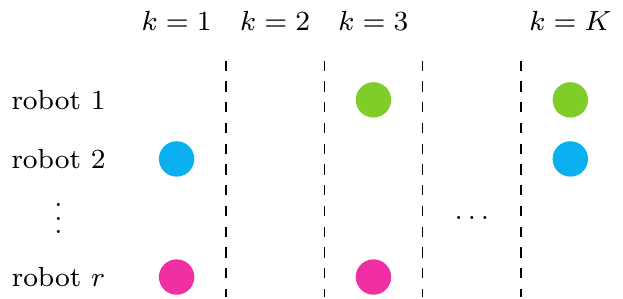}
    \caption{An illustration of the intermittent deployment idea. $r \in \R_2$ and $k=1, \ldots, K$. At time $k=2$, there is no deployment. The constraint \ref{item: M21} and \ref{item: M23} are applied vertically for robots in each time. The constraint \ref{item: M22} is applied horizontally when viewing all robots as a group in each time.}
    \label{fig: demo_intermittent}
\end{figure}

In \figref{fig: demo_intermittent}, we illustrate the intermittent deployment concept. Finally, we must verify that the above constraints are indeed matroidal.

\begin{theorem}
    The constraints $\M_{21}, \M_{22}$ and $\M_{23}$ are matroidal.
\end{theorem}

\begin{proof}
    We only need to verify property \ref{item: matroid2} and \ref{item: matroid3} of \defref{def: matroid} since \ref{item: matroid1} holds by construction.

    % \noindent
    \emph{1) For constraint $\M_{21}$:}

    \emph{For property \ref{item: matroid2}:}
    Consider a set $B_1 \subseteq V$ and assume that for every element $(r, d) \in B_1$ it satisfies that $|B_1 \cap V_k| \le \ell_k$, \ie $B_1 \in \I_{21}$. Now, for any $B_2 \subseteq B_1$ and any element $(r,d) \in B_2$, we know that if $|B_1 \cap V_k| < \ell_k$ then $|B_2 \cap V_k| < \ell_k$. So, if $B_2 \subseteq B_1 \in \I_{21}$, then $B_2 \in \I_{21}$. The property \ref{item: matroid2} is verified.

    \emph{For property \ref{item: matroid3}:}
    Now consider the case $B_1, B_2 \in \I_{21}$. Without loss of generality, we assume that $|B_2| < |B_1|$, \ie, $B_1 \setminus B_2 \neq \emptyset$. Let's assume that there exists no element $e \in (B_1 \setminus B_2)$ such that $B_2 \cup \{e\} \in \I_{21}$. Our assumption implies that $\ell_k$ robots have been allocated in $B_2$, which implies $|B_2| = \ell_k$. However, since $B_1 \in \I_{21}$, it implies $|B_1| \le \ell_k$, we have shown the contradiction as $|B_2| < |B_1|$. So, if $B_1, B_2 \in \I_{22}$ and $|B_2| < |B_1|$, there exists a $e \in (B_1 \setminus B_2)$ such that $B_2 \cup \{e\} \in \I_{21}$. The property \ref{item: matroid3} is verified.

    % \noindent
    \emph{2) For constraint $\M_{22}$:}

    \emph{For property \ref{item: matroid2}:}
    Consider a set $B_1 \subseteq V$ and assume that for every element $(r,d) \in B_1$ it satisfies that $\sum_{k=1}^K \one( |B_1 \cap V_k| ) \le \ell$, \ie $B_1 \in \I_{22}$. For elements $(r,d) \in B_1$, there exist three cases: i) $|B_1 \cap V_k| \ge 1$ holds for every $(r,d) \in B_1$; ii) $|B_1 \cap V_k| = 0$ holds for every $(r,d) \in B_1$; iii) $|B_1 \cap V_k| \ge 1$ for some $(r,d) \in B_1$ and $|B_1 \cap V_k| = 0$ holds for the rest of the elements in $B_1$. \emph{Case i)}: if $|B_1 \cap V_k| \ge 1$, then, either $|B_2 \cap V_k| \ge 1$ holds or $|B_2 \cap V_k| = 0$ holds. Therefore, we know $\sum_{k=1}^K \one( |B_2 \cap V_k| ) \le \ell$ holds. \emph{Case ii)}: If $|B_1 \cap V_k| = 0$, then $|B_2 \cap V_k| = 0$, which implies $\sum_{k=1}^K \one( |B_2 \cap V_k| ) \le \ell$ holds. \emph{Case iii)}: this case is a combination of the first two cases, so $\sum_{k=1}^K \one( |B_2 \cap V_k| ) \le \ell$ still holds. So, if $B_2 \subseteq B_1 \in \I_{22}$, then $B_2 \in \I_{22}$ holds. The property \ref{item: matroid2} is verified.

    \emph{For property \ref{item: matroid3}:}
    Now consider the case where we assume $B_1, B_2 \in \I_{22}$ and $|B_2| < |B_1|$. In this case, we have that $B_1 \setminus B_2$ is non-empty. Let's assume that there exists no element $e \in (B_1 \setminus B_2)$, such that $B_2 \cup \{ e \} \in \I_{22}$. This implies that the number of times where there is at least one deployments for $B_2$ has been reached $\ell$. However, from the definition we know that the number of times where there is at least one deployment for $B_1$ is at most $\ell$. So, we have shown a contradiction as $|B_1| < |S_2|$. So, if $B_1, B_2 \in \I_{22}$ and $|B_2| < |B_1|$, then there exists a $e \in (B_1 \setminus B_2)$ such that $B_2 \cup \{ e \} \in \I_{22}$. The property \ref{item: matroid3} is verified.

    % \noindent
    \emph{3) For constraint $\M_{23}$:}
    The verification is similar to the verification for $\M_{21}$, and we omit it for brevity.
\end{proof}

The intersection of the above constraints forms the matroid intersection constraint $\M_2 = (V, \I_2)$ of this problem with $\I_2 = \I_{21} \cap \I_{22} \cap \I_{23}$. The matroid modeling method is especially useful for robotics applications when constraints are abstract \cite{williams2017decentralized}. Here we use $\M_{21}$, $\M_{22}$, and $\M_{23}$ as examples to illustrate the idea of multi-robot intermittent deployment constraints. Other constraints for this problem can also be integrated easily into the problem formulation, \eg, the number of times for each robot that can be deployed or the composition of the robot teams that are deployed.

It is straightforward to show that such constraints would also obey the matroid properties.

\subsection{The Submodularity of the Objective Function}

After giving the details about each problem, we now focus on the properties of the problem formulation in this section. The second part of the objective function is $h(A) = \max_{B \subseteq V} f(A, B)$, which takes the task allocation's impact into consideration when evaluating the subsequent multi-robot deployment strategy.
If we are interested in the best payoff a robot can experience from the first phase to the second phase, then we have that $f(A, B) = \max_{a \in A} s(a, B)$. To begin with, we make some definitions as follows
\begin{align}
    B_1 & = \argmax_{B \subseteq V} f(X, B), \  & B_3 = \argmax_{B \subseteq V} f(X \cup Y, B), \label{eq: B_1} \\
    B_2 & = \argmax_{B \subseteq V} f(Y, B), \  & B_4 = \argmax_{B \subseteq V} f(X \cap Y, B). \label{eq: B_2}
\end{align}
Following the definition $h(A) = \max_{B \subseteq V} f(A, B)$, we have
\begin{align}
    h(X) & = f(X, B_1), \  & h(X \cup Y) = f(X \cup Y, B_3), \label{eq: h(X)} \\
    h(Y) & = f(Y, B_2), \  & h(X \cap Y) = f(X \cap Y, B_4). \label{eq: h(Y)}
\end{align}
Using these definitions, the properties of the objective function is now formalized.

\begin{theorem}
    If $f(A, B) = \max_{a \in A} s(a, B)$, the objective function $h(A) = \max_{B \subseteq V} f(A, B)$ is non-decreasing and submodular.
    \label{thm: submodularity}
\end{theorem}
\begin{proof}
    \emph{1) Non-decreasing}

    For proving this property, we need to show, for any $X, Y \subseteq E$, if $X \subseteq Y$ then $h(X) \le h(Y)$. When $X \subseteq Y$,
    \begin{equation*}
        \begin{split}
            & h(Y) - h(X) \\
            = & f(Y, B_2) - f(X, B_1) \\
            = & (f(Y, B_2) - f(Y, B_1)) + (f(Y, B_1) - f(X, B_1)).
        \end{split}
    \end{equation*}
    The first equality holds because of \eqref{eq: h(X)} and \eqref{eq: h(Y)}. Following the definition of $B_2$, it holds that $f(Y, B_2) \ge f(Y, B_1)$. Since $X \subseteq Y$ and $f(A, B) = \max_{a \in A} s(a, B)$, it holds that $f(Y, B_1) \ge f(X, B_1)$. Therefore, $h(Y) - h(X) \ge 0$.

    \vskip 0.1in
    \noindent
    \emph{2) Submodularity}

    For proving submodularity, we need to show that for any $X, Y \subseteq E$, the following two hold
    \begin{align}
        \max (f(X, B_1), f(Y, B_2)) = f(X \cup Y, B_3), \label{eq: submodular_max} \\
        \min (f(X, B_1), f(Y, B_2)) \ge f(X \cap Y, B_4). \label{eq: submodular_min}
    \end{align}
    Combining \eqref{eq: submodular_max} and \eqref{eq: submodular_min}, we have $f(X, B_1)+ f(Y, B_2) \ge f(X \cup Y, B_3) + f(X \cap Y, B_4)$.
    This is equivalent to $h(X) + h(Y) \ge h(X \cup Y) + h(X \cap Y)$ and satisfies the submodularity requirement for $h(A)$. So, we only need to prove \eqref{eq: submodular_max} and \eqref{eq: submodular_min}.

    % \bigbreak
    \vskip 0.05in
    \noindent
    \emph{Part a: For proving \eqref{eq: submodular_max}}

    Since any equality $x = y$ can be proven by proving $x \le y$ and $y \le x$ for any $x, y \in \RR$, we will prove \eqref{eq: submodular_max} by proving:
    \begin{align}
        \quad f(X \cup Y, B_3) \le \max (f(X, B_1), f(Y, B_2)), \label{eq: max1} \\
        \quad \max ( f(X, B_1), f(Y, B_2)) \le f(X \cup Y, B_3). \label{eq: max2}
    \end{align}

    \noindent
    \emph{Part a.1: For proving \eqref{eq: max1}}

    From the definition of $B_1$, we have
    \begin{equation*}
        f(X, B_3) \le f(X, B_1) \le \max(f(X, B_1), f(Y, B_2))
        \label{eq: union4}
    \end{equation*}
    These two hold because $B_1 = \argmax_{B \subseteq V} f(X, B)$ and $x \le \max(x, y), \forall x, y \in \RR$. Similarly,
    \begin{equation*}
        f(Y, B_3) \le \max(f(X, B_1), f(Y, B_2)).
        \label{eq: union6}
    \end{equation*}
    We know from the definition that $f(A, B) = \max_{a \in A} s(a, B)$. Then, $f(X \cup Y, B_3) = \max_{a \in X \cup Y} f_1(a, B_3)$. Therefore, we have $f(X \cup Y, B_3) \le \max(f(X, B_1), f(Y, B_2))$ as described in \eqref{eq: max1}.
    This inequality holds because we know that $a \in X \cup Y$ means $a \in X$ or $a \in Y$. If $a \in X$, the first inequality holds. If $a \in Y$, the second inequality holds.

    % \bigbreak
    \vskip 0.05in
    \noindent
    \emph{Part a.2: For proving \eqref{eq: max2}}

    It holds that
    $f(X, B_1) \le f(X \cup Y, B_1) \le f(X \cup Y, B_3).$
    The first inequality holds because of the monotonicity of $f(A, B)$ on $A \subseteq E$ for any $B \subseteq V$, and the second inequality holds because $B_3 = \argmax_{B \subseteq V} f(X \cup Y, B)$. Similarly,
    $f(Y, B_2) \le f(X \cup Y, B_3).$
    Combining these two inequalities, we get the result $\max ( f(X, B_1), f(Y, B_2) ) \le f(X \cup Y, B_3)$.

    % \bigbreak
    \vskip 0.05in
    \noindent
    \emph{Part b: For proving \eqref{eq: submodular_min}}

    It holds that
    $f(X, B_1) \ge f(X, B_4) \ge f(X \cap Y, B_4).$
    The first inequality holds since $B_1 = \argmax_{B \subseteq V} f(X, B)$. The second inequality holds due to the monotonicity of $f(A, B)$ on $A \subseteq E$ for any $B \subseteq V$. Similarly,
    $f(Y, B_2) \ge f(X \cap Y, B_4).$
    Combining these two inequalities, we get the result $\min ( f(X, B_1), f(Y, B_2) ) \ge f(X \cap Y, B_4)$.
\end{proof}

\begin{remark}
    If we are interested in the worst payoff that a robot can experience from the first phase to the second phase, we have $f(A, B) = \min_{a \in A} s(a, B)$. Then, $h(A)$ is non-increasing on $A$.
\end{remark}

It is worth mentioning that the purpose of this letter is to find a method for solving coupled optimization problems in the robotics field, and we use the multi-robot task allocation problem and the multi-robot intermittent deployment problem as examples to illustrate this idea. Other applications can also be applied in this formulation.

\section{Algorithm Analysis}
\label{sec: alg}

\algref{alg} shows the greedy algorithm for solving the coupled problem when both $g(\cdot)$ and $s(\cdot)$ are set functions. If either $g(\cdot)$ or $s(\cdot)$ are a sequence, we only need to change the method slightly. Specifically, if $s(\cdot)$ is a sequence function, we need to change the line 6 in \algref{alg} to $(j \leftarrow 0, \ldots, k \ldots, K-1)$ where $k$ represents the sequence order. We refer to this as a modified version of \algref{alg} for dealing with the sequence function case.

\begin{algorithm}[t]
    \caption{The greedy method for solving the coupled problem}
    \label{alg}
    \textbf{Input:} The inputs are as follows:
    \begin{itemize}
        \item matroid intersection constraints $\M_1$ and $\M_2$;
        \item functions $g(\cdot)$ and $s(\cdot)$.
    \end{itemize}

    \textbf{Output:} Set $A^G$ and set $B^G$.

    \begin{algorithmic}[1]
        % \Statex
        \State $A \leftarrow \emptyset$;
        \For{$i \leftarrow 0, \ldots, |E|-1$} \Comment{\emph{step $i$}} \label{time_i}
        \State{$\B \leftarrow \emptyset$};

        \For{$\forall a \in E \setminus A$ and $A \cup \{a\} \in \I_1$}
        \State $B \leftarrow \emptyset$; \quad $V' \leftarrow \emptyset$;
        \For{$j \leftarrow 0, \ldots, |V|-1$} \Comment{\emph{step $j$}} \label{time_j}
        \State $b' \leftarrow \argmax_{b \in V \setminus V'} f(A \cup \{a\}, B \cup \{b\})$;

        \If{$B \cup \{b'\} \in \I_2$}
        \State $B \leftarrow B \cup \{b'\}$;
        \EndIf

        \State $V' \leftarrow V' \cup \{b'\}$;
        \EndFor
        \State $\B \leftarrow \B \cup \{(a, B)\}$;
        \EndFor

        \If{$\B = \emptyset$} \Comment{\emph{no valid $\B$}}
        \State{\textbf{break}};
        \EndIf

        \State $d(A \cup \{a\}, B) \leftarrow g(A \cup \{a\}) + f(A \cup \{a\}, B)$;
        \State $(a', B^G) \leftarrow \argmax_{\{(a, B)\} \in \B} d(A \cup \{a\}, B)$;
        \State $A \leftarrow A \cup \{a'\}$;
        \EndFor
        \State $A^G \leftarrow A$.
    \end{algorithmic}
\end{algorithm}

\begin{theorem}[\emph{Performance \& complexity}]
    Let $A^G$ and $A^\star$ be greedy and optimal solutions, respectively. $m_1 = |\M_1|$, $m_2 = |\M_2|$. \algref{alg} has the following performance:
    \begin{enumerate}
        \item If $g(\cdot)$ is a non-decreasing modular or submodular set function and
              \begin{itemize}
                  \item if $s(a, B)$ is a non-decreasing modular set function on $B$ for any $a $, then, $m(A^G) \ge 1/(m_2(m_1+1)) m(A^\star)$.
                  \item if $s(a, B)$ is a non-decreasing submodular set function on $B$ for any $a$, then, $m(A^G) \ge 1/((m_1+1)(m_2 + 1)) m(A^\star)$.
              \end{itemize}

        \item If $g(\cdot)$ is a non-decreasing modular or submodular set function and $s(a, B)$ is a non-decreasing sequence submodular function on $B$, then $m(A^G) \ge (m_1+1)^{-1}(1-e^{-1/(m_2+1)}) m(A^\star)$.
        \item \algref{alg} has time complexity $\O(|E|^3 \cdot |V|^2)$.
    \end{enumerate}
\end{theorem}

\begin{proof}
    \emph{1)}
    Let $A^G_i$ denote the greedy output for $A$ at step $i$. In order to get $A^G$, we need to evaluate every $a \notin A^G_i: A^G_i \cup \{a\} \in \I_1$ and incrementally add $a'$ to $A_i$ in terms of maximizing the objective value $m(A^G_i)$. That is, $A_{i+1} = A_i \cup \{a'\}$. In \algref{alg}, we omit the subscript and write it as $A \leftarrow A \cup \{a'\}$ for brevity. This omission also applies to other variables. We cannot evaluate $f(\cdot)$ without knowing $B$ since $f: 2^{E \times V} \mapsto \RR$ and $A \in E, B \in V$. If we can get a $B^\star$ that gives the maximum objective function for each $A^G_i \cup \{a\}$, where $a \notin A^G_i: A^G_i \cup \{a\} \in \I_1$, then it's easy to get $a'$.
    After adding $a$ into $A_i$, we can construct the greedy solution $A_{i+1}$ for $A$ at step $i$, \ie $A^G_{i+1} = A^G_i \cup \{a\}$. Here, $B^\star$ is an optimal solution with respect to every $A^G_i \cup \{a\}$ in terms of objective function value. Due to the intractability for getting $B^\star$ for every $a \notin A^G_i: A^G_i \cup \{a\} \in \I_1$, we propose to use another greedy iteration to get $B^G$ for replacing $B^\star$.

    Now, the only problem left is how to get $B^G$. We construct $B^G_j$ at step $j$ for every $a \notin A^G_i: A^G_i \cup \{a\} \in \I_1$ using a similar greedy method. %through
    Notice that there is a $B^G$ corresponding to every $A^G_i \cup \{a\}$, and the final $B^G$ is corresponding to $A^G_{|E|}$.

    If $g(A)$ is non-decreasing and modular on $A \subseteq E$, then $g(A^G) \ge \frac{1}{m_1} g(A^\star)$; if $g(A)$ is non-decreasing and submodular on $A \subseteq E$, then $g(A^G) \ge \frac{1}{m_1 + 1} g(A^\star)$ \cite{fisher1978analysis}. For the second part of the objective function, it holds that $f(A^G, B^\star) \ge \frac{1}{m_1 + 1} f(A^\star, B^\star)$ since $h(A) = \max_{B \subseteq V} f(A, B)$ is a submodular function on $A$ according to \thmref{thm: submodularity}. Then,
    \begin{itemize}
        \item If $g(A)$ is non-decreasing modular on $A \subseteq E$ and we also use the greedy output $B^G$ for $B$, then $m(A^G) \ge \frac{1}{m_1} g(A^\star) + \frac{1}{m_1 + 1} f(A^\star, B^G)$.
        \item If $g(A)$ is non-decreasing submodular on $A \subseteq E$ and we also use $B^G$ for $B$, then $m(A^G) \ge \frac{1}{m_1 + 1} g(A^\star) + \frac{1}{m_1 + 1} f(A^\star, B^G)$.
    \end{itemize}
    When $s(a, B)$ follows the following property, we obtain
    \begin{itemize}
        \item If $s(a, B)$ is non-decreasing modular on $B \subseteq V$ for any $a \in A$, then $f(A^\star, B^G) \ge \frac{1}{m_2} f(A^\star, B^\star)$.
        \item If $s(A, B)$ is non-decreasing submodular on $B \subseteq V$ for any $a \in A$, then $f(A^\star, B^G) \ge \frac{1}{m_2+1} f(A^\star, B^\star)$.
    \end{itemize}
    Finally, combining one result regarding $f(A,B)$ and one result regarding $m(A)$, we can get a corresponding result.

    \emph{2)}
    The analysis is similar to the analysis for the performance 1) except for a small change regarding how to get $B^G$. When considering $\M_1$ and $\M_2$, we can use the result from \cite{goundan2007revisiting} for analyzing the sequence submodular function $s(\cdot)$ and the above analysis. Then, we have the bound as shown in the statement.

    \emph{3) Computational complexity:}
    To get $A^G$, we need to incrementally select an element $a \in E$ for $\O(|E|)$ times. For each $a \in E$, we need to evaluate its objective function value $\O(|E|)$ times. Also, we need to compute $B^G$ corresponding to  $A \cup \{a\}$. However, since $f(A, B) = \max_{a \in A} s(a, B)$, we need $\O(|E|\cdot |V|^2)$ times for getting $B^G$. The finally computational complexity of \algref{alg} is $\O(|E|^3 \cdot |V|^2)$.
\end{proof}

\section{Simulation Results}
\label{sec: simulations}

In this section, we will demonstrate the performance of the intermittent deployment problem. We will also demonstrate the performance of \algref{alg} by using the combination of the task allocation problem and the intermittent deployment problem as a coupled example.

\subsection{Simulation Setup}

\emph{General Settings:} The general setting is as follows. There is a 2D Gaussian mixture (GMM) environment that needs to be monitored. The task allocation problem and the intermittent deployment problem will operate in this environment in the time order. Different task allocation strategies will have different impacts on the environment, which leads to different initial conditions for the intermittent deployment problem.

\emph{The Task Allocation Problem: }For this problem, we use the modular objective function $g(A) = \sum_{a \in A} u_a$, where $a = (r_1, d, e)$ is an assignment and $u_a$ is the reward for robot $r_1 \in \R_1$ finishing the task $(d, e) \in \O$. For comparison, we set the parameters as follows to make sure that we can get the optimal solution. Specifically, we set the number of robot as $|\R_1| \in \{ 2, \ldots, 6\}$, the requirements cardinality as $|\D| \in \{ 2, \ldots, 6\}$, and functionality cardinality as $|\E| \in \{2, \ldots, 6\}$. In simulation, we use these parameters to generate a random problem instance as our ground set $E$. The reward $u_a$ is generated randomly for all $a \in E$ before conducting the simulations. We use the independence constraint $\M_{11}$ and the uniqueness constraint $\M_{12}$ as the constraints. The problem size of this sub-problem is defined as $|\S_1| = |\R_1| \cdot |\D| \cdot |\E|$.

\emph{The Intermittent Deployment Problem:} For this problem, we have $|\R_2|$ robots available at each time $k$, $k = 1, \ldots, K$, for monitoring this GMM environment along the time horizon $K$. The evolution of the weights of the GMM is modeled as a linear system. That is $x_{k+1} = A x_k + w_k$, where $x_k \in \RR^p$ is the state of the GMM weight and $w_k \in \RR^p$ is the zero-mean Gaussian noise. $p$ is the dimension of the GMM that needs to be estimated. The measurement model at time $k$ is $y_{k+1} = C_k x_k + z_k$, where $z_k \in \RR^q$ is the zero-mean Gaussian noise with noise covariance $Z_k \in \RR^{q \times q}$. Different robots have different measurement abilities, which forms different measurement matrices $C_k$. Now, the intermittent deployment problem becomes how to select robots from $\R_2$ to form the measurement matrix $C_k \in \RR^{q \times p}$ at each time $k$ to maximize the objective function $f(A,B)$. The robots should satisfy the constraints in \secref{ssec: inter}. The objective function is a combination of covariance reduction and the reward. We use the objective function from \cite{jawaid2015submodularity} since it has proven to be sequence submodular under assumptions which will be stated in the following. Because $f(A,B) = \max_{a \in A} s(a, B)$, we only need $s(a, B)$. Specifically, $s(a, B) = \log (\det(P_1(a))/\det(P_K(a))) + \sum_{b \in B} u_b$, where $P_1(a)$ is the GMM weight covariance at time $k=1$. $P_1(a)$ is the crucial connection between the task allocation problem and the intermittent deployment problem. This is because different task allocations result in different $P_1(a)$ and $P_1(a)$ is also the initial condition for the intermittent deployment problem. $P_K(a)$ is the GMM weight covariance at time $K$. That is $P_K(a) = (A^{-\trs})^K P_1(a) A^{-K} + \sum_{k=1}^K (A^{-\trs})^{K-k} M_t A^{-(K-k)}$ \cite{jawaid2015submodularity} where $M_k = C_k^\trs Z_k^{-1} C_k$. $u_b$ is the reward associated with robot $b$. The assumptions for $s(a,B)$ to be sequence submodular is that \cite{jawaid2015submodularity} the state transition matrix $A$ is full rank and $w_t = 0$. Also, $A$ needs to satisfy $AP_1(a)A^\trs \preceq P_1(a)$ and $A^\trs M_k A \preceq M_k$. In the simulation, we use $A = I_p$. In each simulation, both $P_1(a)$ and the reward $u_b$ are also generated randomly for each $b \in V$ before starting the simulations. The dimension $p$ is chosen from the set $\{ 2, \ldots, 5\}$. We set the time horizon as $K = \{2, \ldots, 5\}$ and the number of robots as $|\R_2| = \{2, \ldots, 4\}$. We then run the simulation and greedily select robots in each time. Also, the problem size of this sub-problem is defined as $|\S_2| = |\R_2| \cdot K$.

\subsection{Simulation Performance}

\begin{figure}[!tbp]
    \centering
    \includegraphics[width=0.8\linewidth]{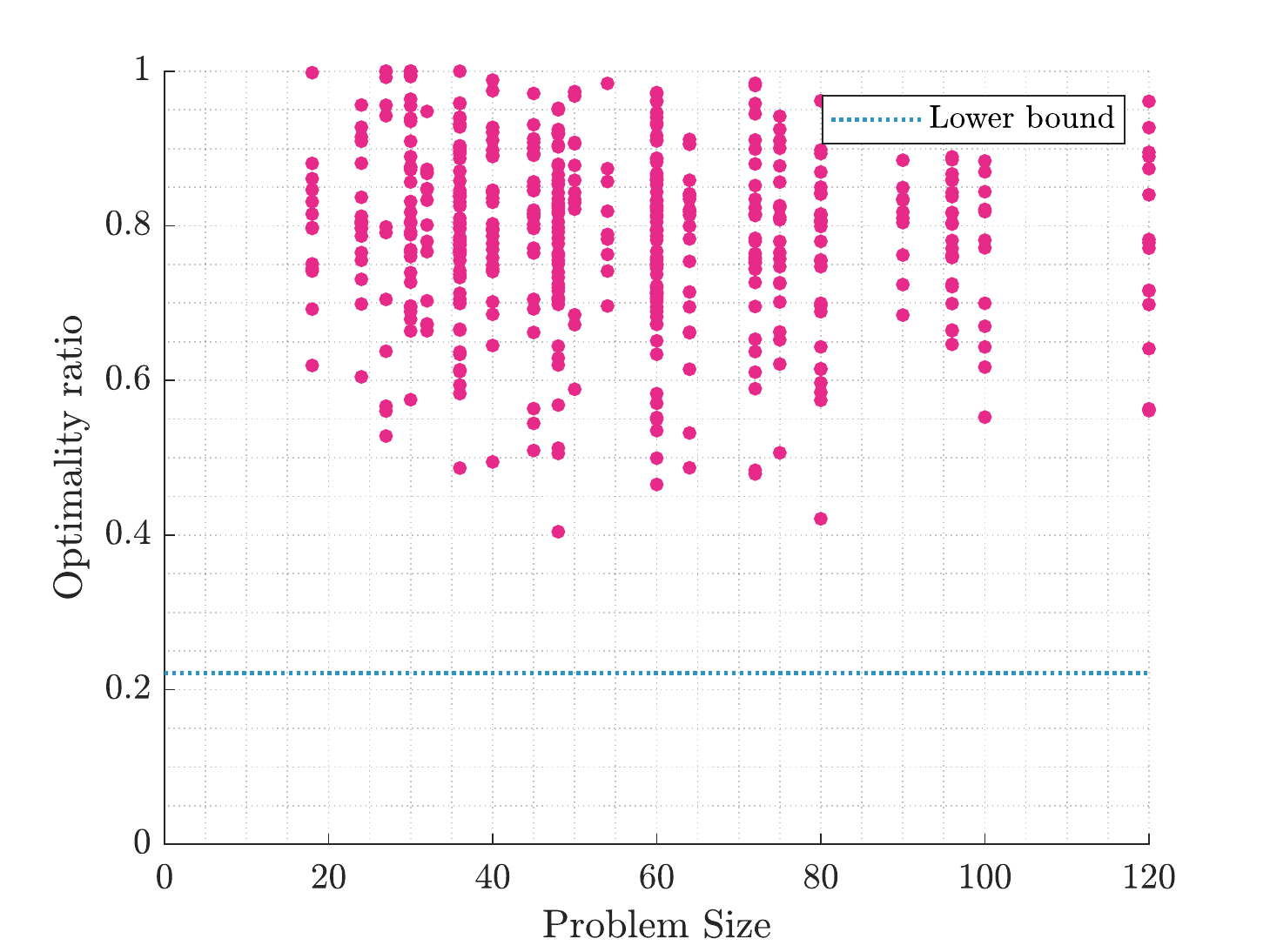}
    \caption{The optimality ratio for the intermittent deployment problem. The problem size is $|\S_2| = |\R_2| \cdot K$.}
    \label{fig: deploy_optimality_ratio}
\end{figure}

\begin{figure}[!tbp]
    \centering
    \includegraphics[width=0.8\linewidth]{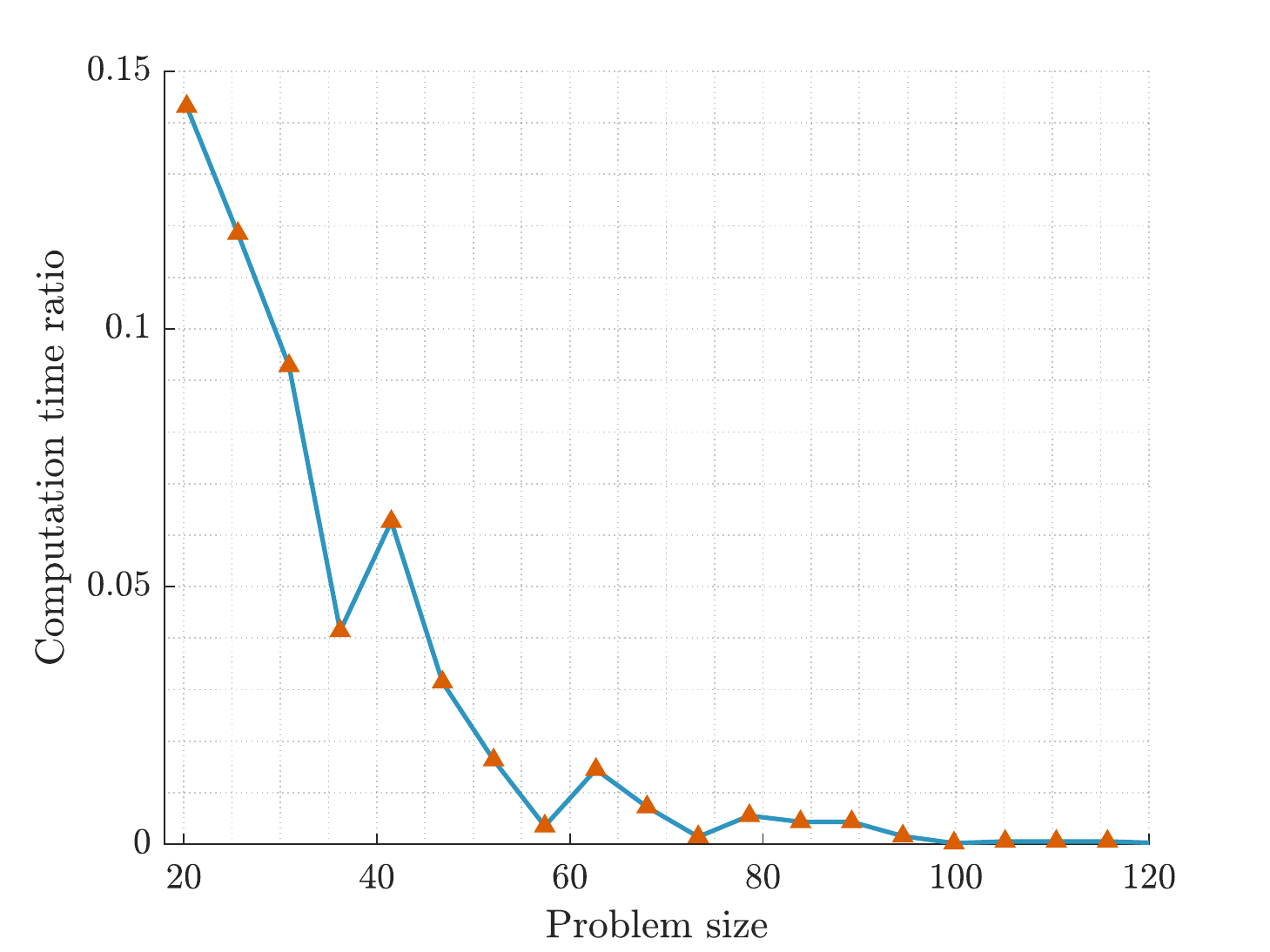}
    \caption{The computation time ratio of the intermittent deployment problem. The problem size is $|\S_2| = |\R_2| \cdot K$.}
    \label{fig: deploy_time_ratio}
\end{figure}

\emph{1) The Intermittent Deployment Problem Performance:}

We first only evaluate the performance of the intermittent deployment problem. To evaluate, we need to know $A$ from the task allocation to get the covariance matrix $P_a$. Here, we use a random covariance matrix $P_a$ as an initial condition. Then, the optimality ratio of the greedy algorithm is shown in \figref{fig: deploy_optimality_ratio}. Also, we define the computation time ratio as $t(B^G)/t(B^\star)$, where $t(G^G)$ is the time for computing the greedy solution and $t(B^\star)$ is the time for computing an optimal solution. We then compute the average computation time ratio for each problem size as shown in \figref{fig: deploy_time_ratio}. We see that the greedy method becomes more efficient as the problem size increases.

\emph{2) The Coupled Problem Performance:}

\emph{Comparison criterion:} We evaluate the performance of the coupled problem. Specifically, we compare the result from the proposed greedy solution with an optimal solution, a heuristic solution, and a random solution. The optimal solution is calculated through the brute force method. A heuristic to solve a coupled problem is to solve each problem separately. We generate the heuristic result using this manner. The random solution is used to demonstrate the effectiveness of the greedy method. The simulation runs 500 times.

\begin{figure*}[htbp]
    \centering
    \subfigure[Optimality ratio: greedy / optimal.]{
        \label{fig: greedy_optimality_ratio}
        \includegraphics[width = .31\linewidth]{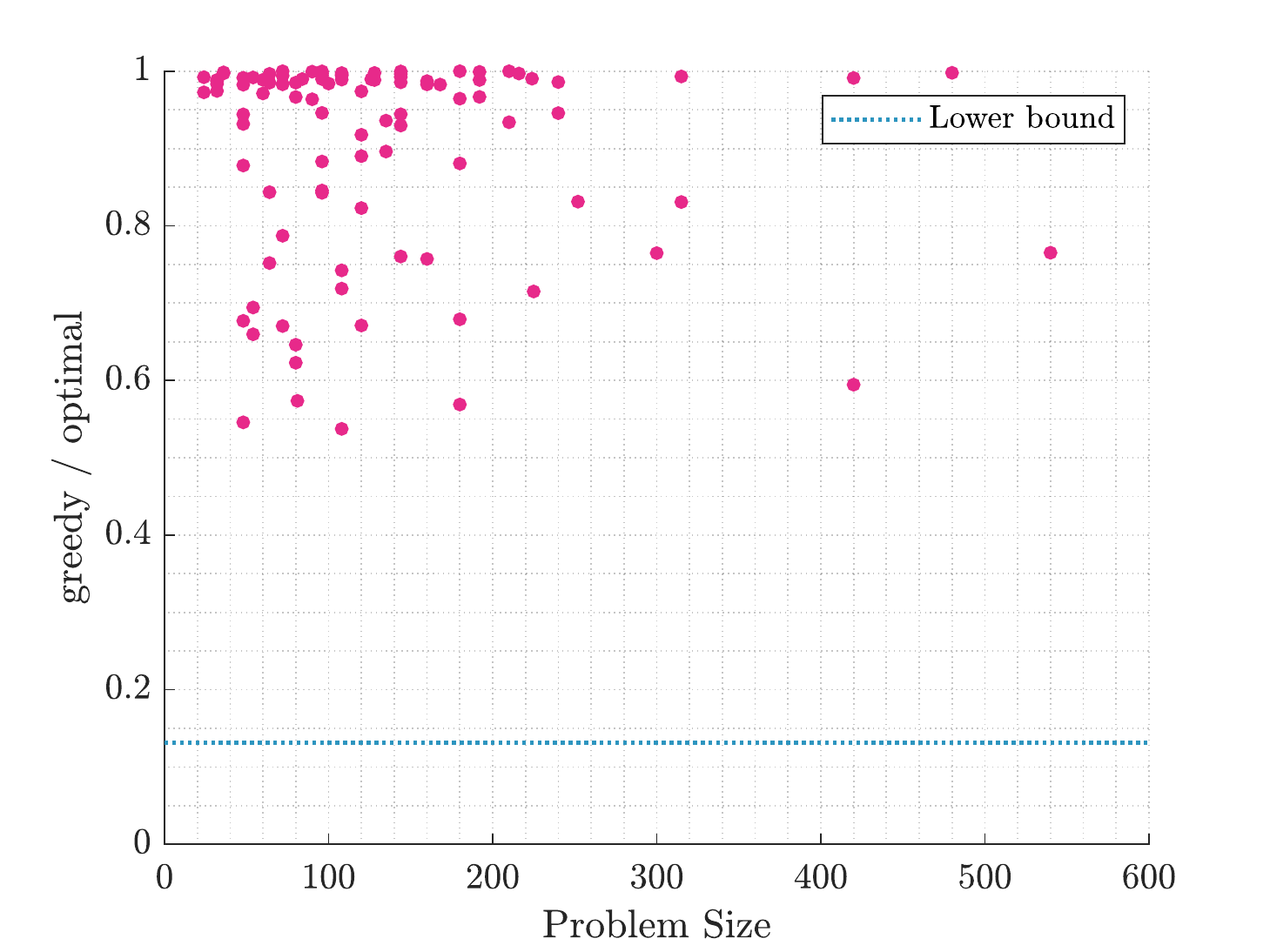}}
    \subfigure[Optimality ratio: heuristic / optimal.]{
        \label{fig: heuristic_optimality_ratio}
        \includegraphics[width = .31\linewidth]{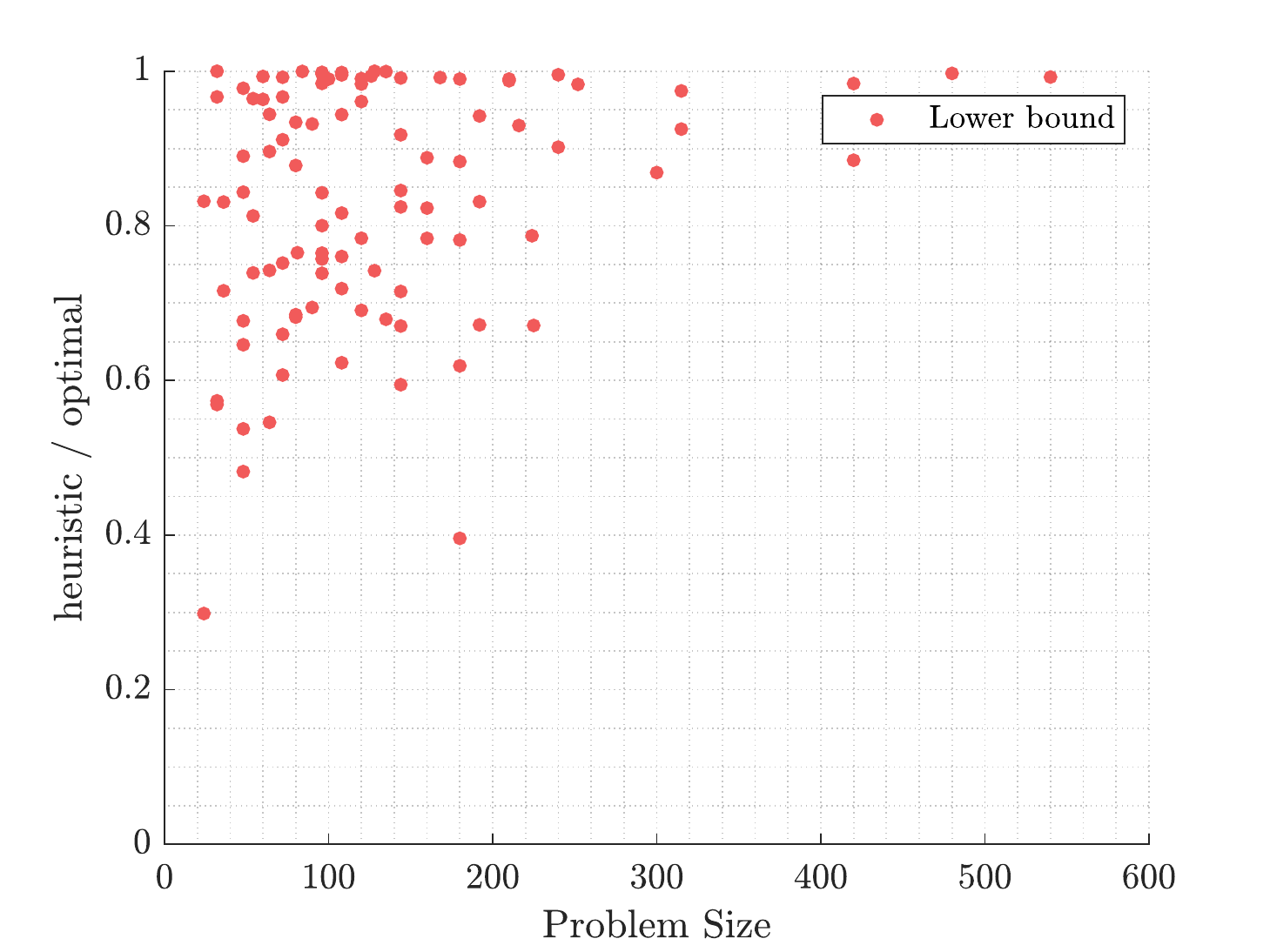}}
    \subfigure[Optimality ratio: random / optimal.]{
        \label{fig: random_optimality_ratio}
        \includegraphics[width = .31\linewidth]{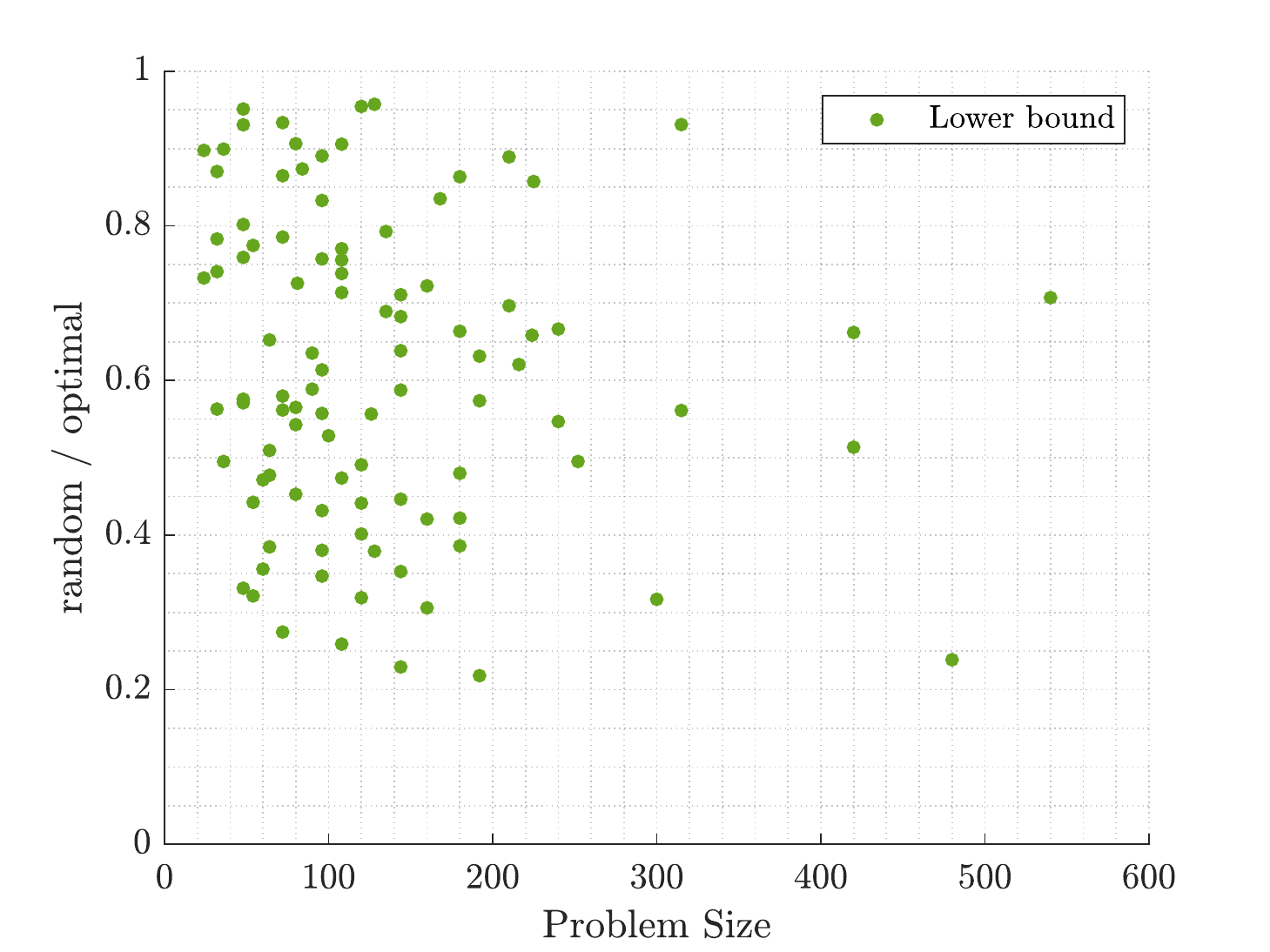}}
    \caption{Monte Carlo simulation performance comparisons: (a) the optimality ratio for the greedy method. (b) the optimality ratio for the heuristic method. (c) the optimality ratio for the random method. The problem size of this coupled problem is $|\S_1| \cdot |\S_2|$.}
\end{figure*}

\emph{Optimality bound:} To characterize the sub-optimality of the proposed greedy algorithm, we calculate the optimality ratio as $m(A^G)/m(A^\star)$. In the simulation, we evaluate the independence matroid constraint $\M_{11}$ and the uniqueness matroid constraint $\M_{12}$ for the multi-robot task allocation problem. For the multi-robot intermittent deployment problem, we use the constraint $\M_{23}$. As shown in \figref{fig: greedy_optimality_ratio}, we observe that the proposed algorithm can generate better optimality ratios than the lower bound in most instances. At the same time, we also plot the result from the heuristic solution in \figref{fig: heuristic_optimality_ratio} and random solution in \figref{fig: random_optimality_ratio}. Further, it occurs often that the optimal solution is found in many instances for the greedy method. To make the comparison more clear, we also calculate the statistics of the optimality ratios of different methods. We calculate the mean and covariance of the optimality ratios for the above three methods. As shown in \tabref{tab: optimality_ratio}, the greedy method has a better performance on average. For the coupled problem, we define the problem size as $|\S_1| \cdot |\S_2|$. Because the complexity of coupled problems increases exponentially as the ground sets size increase, we also observe that even with moderate settings the problem size goes up to $600$, necessitating efficient methods like those proposed in this work.

\begin{table}[!tbp]
    \centering
    \caption{Statistics of the Optimality Ratios of Different Methods}
    \label{tab: optimality_ratio}
    \begin{tabular}{ccc}
        \toprule
        method           & mean        & covariance   \\
        \midrule
        \emph{greedy}    & \emph{0.89} & \emph{0.018} \\
        \emph{heuristic} & 0.83        & 0.024        \\
        \emph{random}    & 0.61        & 0.040        \\
        \bottomrule
    \end{tabular}
\end{table}

\section{Conclusions and Future Work}
\label{sec: conclusions}

In this letter, we presented a method for optimizing coupled problems by using submodular optimization. We demonstrated how to solve the proposed coupled problem by using the task allocation problem and the intermittent deployment problem as motivating examples. From the constraint perspective, we illustrated how to model general constraints as matroid constraints. From the objective function perspective, we demonstrated under which conditions the objective function is (sub)modular, which indicates the existence of an effective greedy algorithm. At the same, we analyzed the performance and computational complexity of our algorithm. In the end, Monte Carlo simulations demonstrated the effectiveness of the proposed algorithm. A direction for future work is to exploit more efficient algorithms with tighter bounds because most of the results achieve better performance in practice than the theoretical optimality bound.

\bibliographystyle{IEEEtran}
\bibliography{ref}

\end{document}